\documentclass{aamas2012}
\usepackage{amsthm}

\newtheorem{theorem}{Theorem}
\newtheorem{proposition}[theorem]{Proposition}
 
\begin{document}

% In the original styles from ACM, you would have needed to
% add meta-info here. This is not necessary for AAMAS 2012  as
% the complete copyright information is generated by the cls-files.

\title{Learning Performance of Prediction Markets with Kelly Bettors}

% AUTHORS

% For initial submission, do not give author names, but the
% tracking number, instead, as the review process is blind.


%\numberofauthors{8} %  in this sample file, there are a *total*

\numberofauthors{3}

\author{
% 1st. author
\alignauthor
Alina Beygelzimer\\ 
\affaddr{IBM Research}\\
\email{beygel@us.ibm.com}
\alignauthor
John Langford\\
\affaddr{Yahoo! Research}\\
\email{jl@yahoo-inc.com}
\alignauthor
David Pennock\\
\affaddr{Yahoo! Research}\\
\email{pennockd@yahoo-inc.com}
%Ben Trovato\titlenote{Dr.~Trovato insisted his name be first.}\\
%       \affaddr{Institute for Clarity in Documentation}\\
%       \affaddr{1932 Wallamaloo Lane}\\
%       \affaddr{Wallamaloo, New Zealand}\\
%       \email{trovato@corporation.com}
% 2nd. author
%\alignauthor
%G.K.M. Tobin\titlenote{The secretary disavows any knowledge of this author's actions.}\\
%       \affaddr{Institute for Clarity in Documentation}\\
%       \affaddr{P.O. Box 1212}\\
%       \affaddr{Dublin, Ohio 43017-6221}\\
%       \email{webmaster@marysville-ohio.com}
% 3rd. author
%\alignauthor Lars Th{\o}rv{\"a}ld\titlenote{This author is the one who did all the really hard work.}\\
%       \affaddr{The Th{\o}rv{\"a}ld Group}\\
%       \affaddr{1 Th{\o}rv{\"a}ld Circle}\\
%       \affaddr{Hekla, Iceland}\\
%       \email{larst@affiliation.org}
}

\maketitle

\begin{abstract}
In evaluating prediction markets (and other crowd-prediction mechanisms), investigators have repeatedly observed a so-called \emph{wisdom of crowds} effect, which can be roughly summarized as follows: the average of participants performs much better than the average participant. The market price---an average or at least aggregate of traders' beliefs---offers a better estimate than most any individual trader's opinion. In this paper, we ask a stronger question: how does the market price compare to the \emph{best} trader's belief, not just the average trader. We measure the market's worst-case log regret, a notion common in machine learning theory. To arrive at a meaningful answer, we need to assume something about how traders behave. We suppose that every trader optimizes according to the \emph{Kelly criteria}, a strategy that provably maximizes the compound growth of wealth over an (infinite) sequence of market interactions. We show several consequences. First, the market prediction is a wealth-weighted average of
the individual participants' beliefs. Second, the market
\emph{learns} at the optimal rate, the market price reacts exactly as if updating according to Bayes' Law, and the market prediction has low worst-case
log regret to the best individual participant. We simulate a sequence of markets where an underlying true
probability exists, showing that the market converges to the true objective
frequency as if updating a Beta distribution, as the theory predicts.
If agents adopt a fractional Kelly criteria, a common practical variant, we show that agents behave like full-Kelly agents with beliefs weighted between their own and the market's, and that the market price converges to a time-discounted frequency.
Our analysis provides a new justification for fractional Kelly betting,
a strategy widely used in practice for ad-hoc reasons. Finally, we propose a
method for an agent to learn her own optimal Kelly fraction.
\end{abstract}

% Note that the category section should be completed after reference to the ACM Computing Classification Scheme available at
% http://www.acm.org/about/class/1998/.


%A category including the fourth, optional field follows...
%\category{D.2.8}{Software Engineering}{Metrics}[complexity measures, performance measures]

%General terms should be selected from the following 16 terms: Algorithms, Management, Measurement, Documentation, Performance, Design, Economics, Reliability, Experimentation, Security, Human Factors, Standardization, Languages, Theory, Legal Aspects, Verification.


%Keywords are your own choice of terms you would like the paper to be indexed by.


\section{Introduction}
Consider a gamble on a binary event, say, that Obama will win the 2012
US Presidential election, where every $x$ dollars risked earns $x b$ dollars in
net profit if the gamble pays off. How many dollars $x$ of your wealth should you
risk if you believe the probability is $p$? The gamble is favorable if $b p - (1-p) > 0$, in which case betting your entire wealth $w$ will maximize your expected profit. However, that's extraordinarily risky: a single stroke of bad luck loses everything. Over the course of many such gambles, the probability of bankruptcy approaches 1. On the other hand, betting a small fixed amount avoids bankruptcy but cannot take advantage of compounding growth.  %This is nontrivial---if you risk every dollar, then a mistake
%implies losing everything.  And if you risk nothing, then no gain can
%be made.

The \emph{Kelly criteria} prescribes choosing $x$ to maximize
the expected compounding growth rate of wealth, or
equivalently to maximize the expected logarithm of wealth. Kelly betting is asymptotically optimal, meaning that in the limit over many gambles, a Kelly bettor will grow wealthier than an otherwise identical non-Kelly bettor with probability 1 \cite{Breiman1961,Cover2006,Kelly1956,Thorp1969,Thorp1997}.

Assume all agents in a market optimize according to the Kelly
principle, where $b$ is selected to clear the market. We consider the
implications for the market as a whole and properties of the market
odds $b$ or, equivalently, the market probability $p_m = 1/(1+b)$. We
show that the market prediction $p_m$ is a wealth-weighted average of
the agents' predictions $p_i$. Over time, the market itself---by
reallocating wealth among participants---adapts at the optimal rate
with bounded log regret to the best individual agent.  When a true
objective probability exists, the market converges to it as if
properly updating a Beta distribution according to Bayes' rule.  These
results illustrate that there is no ``price of anarchy'' associated
with well-run prediction markets.

We also consider fractional Kelly betting, a lower-risk variant of Kelly betting that is popular in practice but has less theoretical grounding. We provide a new justification for fractional Kelly based on agent's confidence. In this case, the market prediction is a confidence-and-wealth-weighted average that empirically converges to a time-discounted version of objective frequency.
Finally, we propose a method for agents to learn their optimal fraction over time.

\section{Kelly betting}\label{sec:frac-kelly}

When offered $b$-to-1 odds on an event with probability $p$, the Kelly-optimal amount to bet is $f^{*} w$, where
\[
f^{*}=\frac{bp-(1-p)}{b}\]
is the optimal fixed fraction of total wealth $w$ to commit to the gamble.

If $f^{*}$ is negative, Kelly says to avoid betting: expected profit is negative. If $f^{*}$ is positive, you have an information edge; Kelly
says to invest a fraction of your wealth proportional to how
advantageous the bet is. In addition to maximizing the growth rate of
wealth, Kelly betting maximizes the geometric mean of wealth and asymptotically minimizes the
mean time to reach a given aspiration level of wealth \cite{Thorp1997}.

Suppose fair odds of $1/b$ are simultaneously offered on the opposite outcome (e.g., Obama will \emph{not} win the election). If $bp-(1-p)<0$, then betting on this opposite outcome is favorable; substituting $1/b$ for $b$ and $1-p$ for $p$, the optimal fraction of wealth to bet becomes $1-p-bp$.

An equivalent way to think of a gamble with odds $b$ is as a prediction market with price $p_m=1/(1+b)$. The volume of bet is specified by choosing a quantity $q$ of \emph{shares}, where each share is worth \$1 if the outcome occurs and nothing otherwise. The price represents the cost of one share: the amount needed to pay for a chance to win back \$1. In this interpretation, the Kelly formula becomes
\[
f^{*}=\frac{p-p_m}{1-p_m}.\]
The optimal action for the agent is to trade $q^{*}=f^{*} w / p_m$ shares, where $q^{*}>0$ is a buy order and $q^{*}<0$ is a sell order, or a bet against the outcome.

Note that $q^{*}$ is the optimum of expected log utility 
\[
p\ln((1-p_m)q+w)+(1-p)\ln(-p_m q+w).\] This is not a coincidence: Kelly betting is identical to maximizing expected log utility.

\section{Market model}

Suppose that we have a prediction market,
where participant $i$ has a starting wealth $w_i$ with $\sum_i w_i = 1$.
Each participant $i$ uses Kelly betting to determine the
fraction $f^*_i$ of their wealth bet,
depending on their predicted probability $p_i$.

We model the market as an auctioneer matching supply and demand, taking no profit and absorbing no loss. We adopt a competitive equilibrium concept, meaning that agents are "price takers", or do not consider their own effect on prices if any. Agents optimize according to the current price and do not reason further about what the price might reveal about the other agents' information. An exception of sorts is the fractional Kelly setting, where agents do consider the market price as information and weigh it along with their own.

A market is in competitive equilibrium at price $p_m$ if all agents
are optimizing and $\sum_i q_i^{*}=0$, or every buy order and sell
order are matched.  We discuss next what the value of $p_m$ is.

\section{Market prediction}

In order to define the prediction market's performance, we must define
its prediction $b$, or the equilibrium payoff odds reached when all
agents are optimizing, and supply and demand are precisely balanced.
Recall that the market's probability implied by the odds of $b$ is
$p_m = 1/(1+b)$.  We will show that $p_m$ is $\sum_{i}w_{i}p_{i}$.

\subsection{Payout balance}
The first approach we'll use is payout balance: the amount of money at
risk must be the same as the amount paid out.

\begin{theorem}\label{thm:pricing}(Market Pricing)  For all normalized agent wealths $w_i$ and agent beliefs $p_i$,
$$ p_m = \sum_i p_i w_i $$
\end{theorem}
\begin{proof}
To see this, recall that $f_{i}^{*}=(p_i-p_m)/(1-p_m)$ for
$p_{i}>p_m$. For $p_{i}<p_m$, Kelly betting prescribes taking the
other side of the bet, with fraction
\[
\frac{(1-p_{i})-(1-p_m)}{1-(1-p_m)}=\frac{p_m-p_{i}}{p_m}.\] So the
market equilibrium occurs at the point $p_m$ where the payout is equal
to the payin.  If the event occurs, the payin is
\[
(1 + b)\sum_{i:p_{i}>p_m}\frac{p_i-p_m}{1-p_m}w_{i} = 
\frac{1}{p_m}\sum_{i:p_{i}>p_m}\frac{p_i-p_m}{1-p_m}w_{i}.
\]
Thus we want
\begin{align*}
\frac{1}{p_m}\sum_{i:p_{i}>p_m}\frac{p_i-p_{m}}{1-p_m}w_{i} & =
\sum_{i:p_{i}>p_m}\frac{p_{i}-p_m}{1-p_m}w_{i} \ + \\
& \qquad \: \sum_{i:p_{i}<p_m}\frac{p_m-p_{i}}{p_m}w_{i}, \quad \text{or}\\
\frac{1-p_m}{p_m}\sum_{i:p_{i}>p_m}\frac{p_i-p_{m}}{1-p_m}w_{i} & =\sum_{i:p_{i}< p_m}\frac{p_{m}-p_i}{p_m}w_{i}, \quad \text{or}\\
\sum_{i:p_{i}>p_m}(p_i-p_{m})w_{i} & =\sum_{i:p_{i}<p_m}(p_{m}-p_i)w_{i},
\quad \text{or} \\
\sum_{i}p_i w_{i} & = \sum_{i}p_{m}w_{i}.
\end{align*}
\noindent
Using $\sum_{i}w_{i}=1$, we get the theorem.
\end{proof}

\subsection{Log utility maximization}

An alternate derivation of the market prediction utilizes the fact that Kelly betting is equivalent to maximizing expected log utility. Let $q=x(b+1)$ be the gross profit of an agent who risks $x$ dollars, or in prediction market language the number of shares purchased. Then expected log utility is
\begin{displaymath}
E[U(q)] = p\ln((1-p_m)q+w)
+(1-p)\ln(-p_m q+w).
\end{displaymath}
The optimal $q$ that maximizes $E[U(q)]$ is
\begin{equation}
q(p_m) = \frac{w}{p_m} \cdot \frac{p-p_m}{1-p_m}.
\label{eq:linop-1-dem}
\end{equation}

\begin{proposition}
In a market of agents each with log utility and initial wealth $w$, the competitive equilibrium price is
\begin{equation}
p_m = \sum_i w_i p_i
\label{eq:linop-price}
\end{equation}
where we assume $\sum_{i}w_{i}=1$, or $w$ is normalized wealth not absolute wealth.
\label{thm:market-linop}
\end{proposition}
\textbf{Proof.} These prices satisfy $\sum_i q_i = 0$, the condition for competitive equilibrium (supply equals demand), by substitution. $\Box$ \smallskip

This result can be seen as a simplified derivation of that by Rubinstein
\cite{Rubinstein74,Rubinstein75,Rubinstein76} and is also discussed by Pennock and Wellman \cite{Pennock01-mfpo-tr,Pennock99-thesis} and Wolfers and Zitzewitz \cite{Wolfers2006}.

\section{Learning Prediction Markets}

Individual participants may have varying prediction qualities and
individual markets may have varying odds of payoff.  What happens to
the wealth distribution and hence the quality of the market prediction
over time?  We show next that the market \emph{learns} optimally for
two well understood senses of optimal.

\subsection{Wealth redistributed according to Bayes' Law}

In an individual round, if an agent's belief is $p_{i}>p_m$, then they
bet $\frac{p_{i}-p_m}{1-p_m}w_{i}$ and have a total wealth afterward
dependent on $y$ according to:

\begin{align*}
\text{If}\quad y=1, & \quad \left(\frac{1}{p_m}-1\right)\frac{p_{i}-p_m}{1-p_m}w_{i}+w_{i}=\frac{p_{i}}{p_m}w_{i} \\ 
\text{If}\quad y=0, & \quad (-1)\frac{p_{i}-p_m}{1-p_m}w_{i}+w_{i}=\frac{1-p_{i}}{1-p_m}w_{i}
\\
\end{align*}

Similarly if $p_{i}<p_m$, we get: 
\begin{align*}
\text{If}\quad y=1, &\quad (-1)\frac{p_m-p_{i}}{p_m}w_{i}+w_{i}=\frac{p_{i}}{p_m}w_{i} \\
\text{If}\quad y=0, &\quad \left(\frac{1}{1-p_m}-1\right)\frac{p_m-p_{i}}{p_m}w_{i}+w_{i}=\frac{1-p_{i}}{1-p_m}w_{i}, \end{align*}
which is identical. 

%Bayes Law states that $P(A|B) = \frac{P(B|A)P(A)}{P(B)}$. 
If we treat the prior probability that agent $i$ is correct as $w_i$, 
Bayes' law states that the posterior probability of choosing 
agent $i$ is
$$
P(i\mid y=1) = \frac{P(y=1\mid i)P(i)}{P(y=1)} = \frac{p_i w_i}{p_m} = \frac{p_i w_i}{\sum_i p_i w_i},
$$ which is precisely the wealth computed above for the $y=1$ outcome.
The same holds when $y=0$, and so Kelly bettors redistribute
wealth according to Bayes' law.

\subsection{Market Sequences}
\label{sec:sequence}

It is well known that Bayes' law is the correct approach for
integrating evidence into a belief distribution, which shows that
Kelly betting agents optimally summarize all past information if the
true behavior of the world was drawn from the prior distribution of
wealth.

Often these assumptions are too strong---the world does not behave
according to the prior on wealth, and it may act in a manner
completely different from any one single expert.  In that case, a
standard analysis from learning theory shows that the market has
low \emph{regret}, performing almost as well as the best market
participant.

For any particular sequence of markets we have a sequence $p_{t}$
of market predictions and $y_{t}\in\{0,1\}$ of market outcomes. We
measure the accuracy of a market according to log loss as \[
L\equiv\sum_{t=1}^{T}I(y_{t}=1)\log\frac{1}{p_{t}}+I(y_{t}=0)\log\frac{1}{1-p_{t}}.\]
Similarly, we measure the quality of market participant making prediction
$p_{it}$ as \[
L_{i}\equiv\sum_{t=1}^{T}I(y_{t}=1)\log\frac{1}{p_{it}}+I(y_{t}=0)\log\frac{1}{1-p_{it}}.\]
So after $T$ rounds, the total wealth of player
$i$ is \[
w_{i}\prod_{t=1}^{T}\left(\frac{p_{it}}{p_{t}}\right)^{y_{t}}\left(\frac{1-p_{it}}{1-p_{t}}\right)^{1-y_{t}},\]
where $w_i$ is the starting wealth.
We next prove a well-known theorem for learning in the present context (see for example~\cite{FSSW1997}).
\begin{theorem}For all sequences of participant predictions
$p_{it}$ and all sequences of revealed outcomes $y_{t}$, \[
L\leq\min_{i}L_{i}+\ln\frac{1}{w_{i}}.\]
\end{theorem}
This theorem is extraordinarily general, as it applies to \emph{all}
market participants and \emph{all} outcome sequences, even when these
are chosen adversarially.  It states that even in this worst-case
situation, the market performs only $\ln 1/w_i$ worse than the best
market participant $i$.  
\begin{proof}
Initially, we have that $\sum_{i}w_{i}=1$. After $T$ rounds, the total wealth
of any participant $i$ is given by 
\[
w_{i}\prod_{t=1}^{T}\left(\frac{p_{it}}{p_{t}}\right)^{y_{t}}\left(\frac{1-p_{it}}{1-p_{t}}\right)^{1-y_{t}}=w_{i}e^{L-L_{i}}\leq1,\]
where the last inequality follows from wealth being conserved.
Thus
$
\ln w_{i}+L-L_{i}\leq0$, yielding
\[
L\leq L_{i}+\ln\frac{1}{w_{i}}.\]
\end{proof}

\section{Fractional Kelly Betting}
\emph{Fractional Kelly betting} says to invest a smaller fraction $\lambda f^{*}$ of wealth for $\lambda < 1$. Fractional Kelly is usually justified on an ad-hoc basis as either (1) a risk-reduction strategy, since practitioners often view full Kelly as too volatile, or (2) a way to protect against an inaccurate belief $p$, or both  \cite{Thorp1997}.
%Seems untrue --John (For example, with full Kelly, at any given time there is a 1 in $k$ chance of losing all but $1/k$ of your wealth.) 
Here we derive an alternate interpretation of fractional Kelly. In prediction market terms, the fractional Kelly formula is
\[
\lambda\frac{p-p_m}{1-p_m}.\]
With some algebra, fractional Kelly can be rewritten as
\[
\frac{p'-p_m}{1-p_m}\]
where 
\begin{equation}
p' = \lambda p + (1-\lambda) p_m .
\label{eq:update-glu}
\end{equation}
In other words, $\lambda$-fractional Kelly is precisely equivalent to full Kelly with revised belief $\lambda p + (1-\lambda)p_m$, or a weighted average of the agent's original belief and the market's belief. In this light, fractional Kelly is a form of confidence weighting where the agent mixes between remaining steadfast with its own belief ($\lambda=1$) and acceding to the crowd and taking the market price as the true probability ($\lambda=0$). The weighted average form has a Bayesian justification if the agent has a Beta prior over $p$ and has seen $t$ independent Bernoulli trials to arrive at its current belief. If the agent envisions that the market has seen $t'$ trials, then she will update her belief to $\lambda p + (1-\lambda)p_m$, where $\lambda = t/(t+t')$  \cite{Morris83,Pennock99-thesis,Rosenblueth92}. The agent's posterior probability
given the price is a weighted average of its prior and the price, where the weighting term captures her perception of her own confidence, expressed in terms of the independent observation count seen as compared to the market.

\section{Market prediction with fractional Kelly}

When agents play fractional Kelly, the competitive equilibrium price naturally changes.  The resulting market price is easily compute, as for fully Kelly agents.

\begin{theorem}(Fractional Kelly Market Pricing)  For all agent beliefs $p_i$, normalized wealths $w_i$ and fractions $\lambda_i$ 
\begin{equation}
p_m = \frac{\sum_i \lambda_i w_i p_i}{\sum_l \lambda_l w_l}  .
\label{eq:linop-learn-price}
\end{equation}
\end{theorem}
Prices retain the form of a weighted average, but with weights proportional to
the product of wealth and self-assessed confidence.

\begin{proof}
The proof is a straightforward corollary of Theorem~\ref{thm:pricing}.  In particular, we note that a $\lambda$-fractional Kelly agent of wealth $w$ bets precisely as a full-Kelly agent of wealth $\lambda w$.  Consequently, we can apply theorem~\ref{thm:pricing} with $w_i' = \frac{\lambda_i w_i}{\sum_i \lambda_i w_i}$ and $p_i' = p_i$ unchanged.
\end{proof}

\section{Market dynamics with stationary objective frequency}

The worst-case bounds above hold even if event outcomes are chosen by a malicious adversary. In this section, we examine how the market performs when the objective frequency of outcomes is unknown though stationary.

The market consists of a single bet repeated over the course of $T$
periods. Unbeknown to the agents, each event unfolds as an
independent Bernoulli trial with probability of success $\pi$. At the
beginning of time period $t$, the realization of event $E_t$ is
unknown and agents trade until equilibrium. Then the outcome is
revealed, and the agents' holdings pay off accordingly. As time period
$t+1$ begins, the outcome of $E_{t+1}$ is uncertain. Agents bet on the
$t+1$ period event until equilibrium, the outcome is revealed, payoffs
are collected, and the process repeats.

In an economy of Kelly bettors, the equilibrium price is a
wealth-weighted average (\ref{eq:linop-price}). Thus, as an agent
accrues relatively more earnings than the others, its influence on
price increases. In the next two subsections, we examine how this
adaptive process unfolds; first, with full-Kelly agents and second, with fractional Kelly agents. In the former case, prices react exactly as if the market were a single agent updating a Beta distribution according to Bayes' rule.

\subsection{Market dynamics with full-Kelly agents}

\begin{figure}
(a)\includegraphics[scale=0.9]{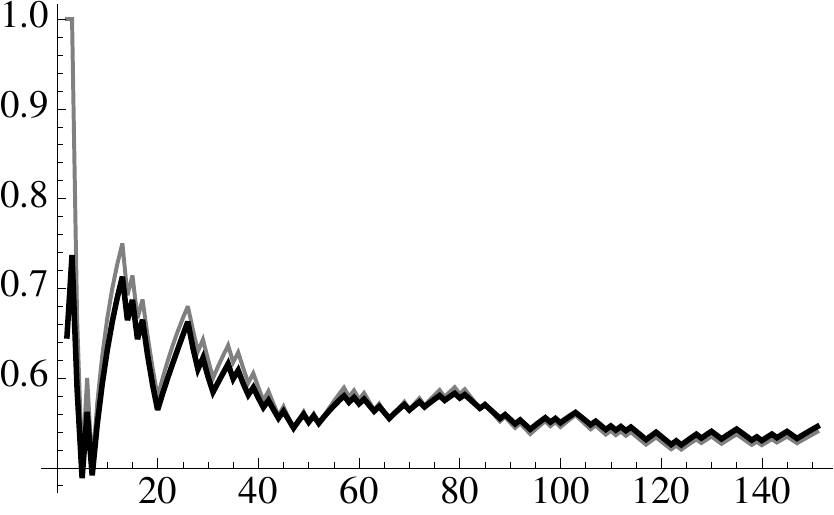} 
(b)\includegraphics[scale=0.9]{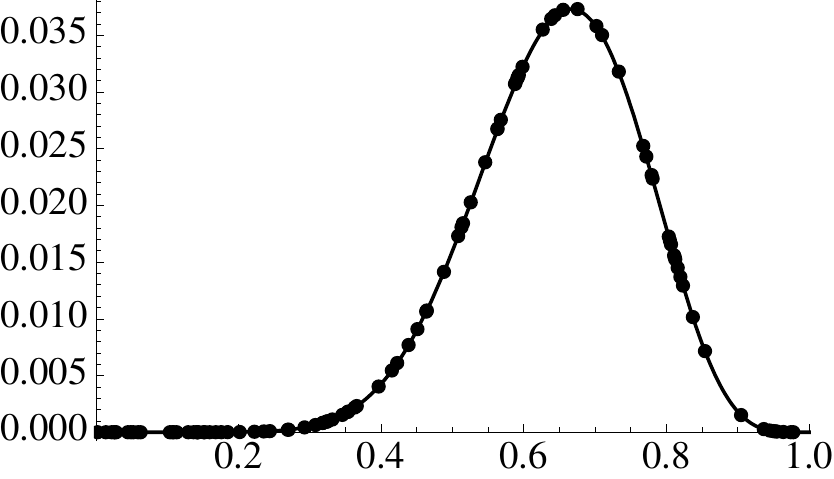}
\caption
%[Price and observed frequency over time for fixed-belief agents with GLU.]
{(a) Price (black line) versus the
observed frequency (gray line) of the event over 150 time periods.
The market consists of 100 full-Kelly agents with
initial wealth $w_i=1/100$. (b) Wealth after 15 time periods versus belief for 100 Kelly agents. The event has
occurred in 10 of the 15 trials. The solid line is the
posterior Beta distribution consistent with observing 10
successes in 15 independent Bernoulli trials.}
\label{fig:adapt-price-glu}
\end{figure}
Figure~\ref{fig:adapt-price-glu}.a plots the price over
150 time periods, in a market composed of 100 Kelly agents with initial wealth $w_i=1/100$, and $p_i$ generated
randomly and uniformly on $(0,1)$. In this simulation the true probability of success $\pi$ is $0.5$. For
comparison, the figure also shows the \emph{observed frequency}, or the
number of times that $E$ has occurred divided by the number
of periods. The market price tracks the observed frequency extremely
closely. Note that price changes are due entirely to a transfer of
wealth from inaccurate agents to accurate agents, who then wield more
power in the market; individual beliefs remain fixed.

Figure~\ref{fig:adapt-price-glu}.b illustrates the nature of this
wealth transfer. The graph provides a snapshot of agents' wealth
versus their belief $p_i$ after period 15. In this run, $E$ has
occurred in 10 out of the 15 trials.  The maximum in wealth is
near 10/15 or 2/3. The solid line in the figure is a Beta distribution with
parameters $10+1$ and $5+1$. This
distribution is precisely the posterior probability of success that
results from the observation of 10 successes out of 15 independent
Bernoulli trials, when the prior probability of success is uniform on
(0,1).  The fit is essentially perfect, and can be proved in the limit
since the Beta distribution is conjugate to the Binomial distribution under Bayes' Law.

Although individual agents are not adaptive, the market's composite
agent computes a proper Bayesian update. Specifically, wealth
is reallocated proportionally to a Beta distribution corresponding to
the observed number of successes and trials, and price is approximately
the expected value of this Beta distribution.\footnote{As $t$ grows,
this expected value rapidly approaches the observed frequency plotted
in Figure~\ref{fig:adapt-price-glu}.} Moreover, this correspondence
holds regardless of the number of successes or failures, or the
temporal order of their occurrence. A kind of collective Bayesianity
\emph{emerges} from the interactions of the group.

We also find empirically that, even if not all agents are Kelly bettors, among those that are, wealth is still redistributed according to Bayes' rule.

\subsection{Market dynamics with fractional Kelly agents}

\begin{figure}
(a)\includegraphics[scale=0.9]{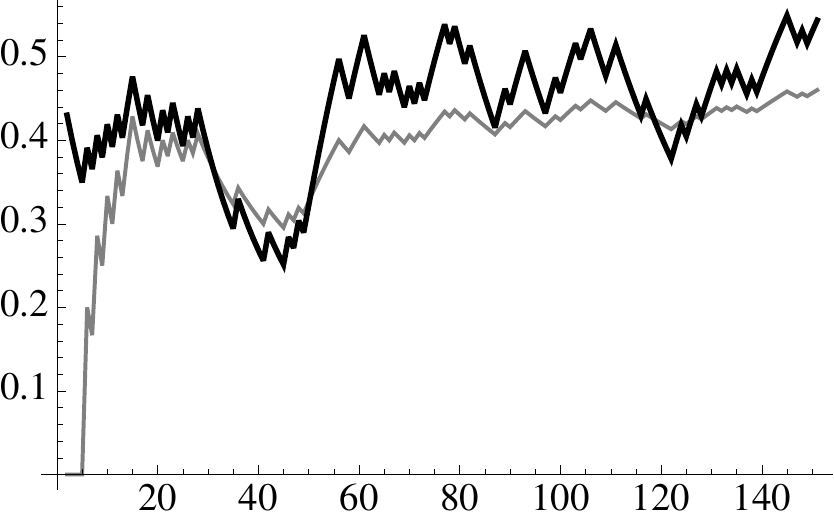}\\
(b) \includegraphics[scale=0.9]{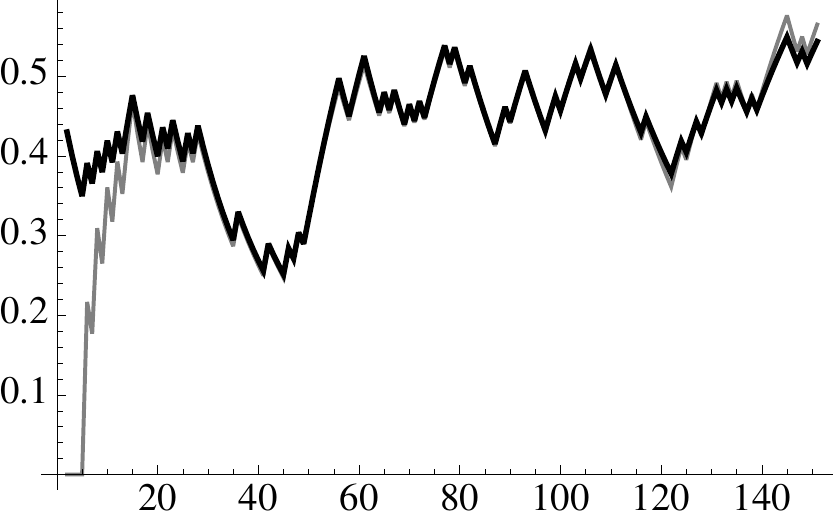}
% $s = 74$
\caption{(a) Price (black line) versus observed frequency
(gray line) over 150 time periods for 100 agents with Kelly fraction $\lambda=0.2$. As the frequency converges to
$\pi=0.5$, the price remains volatile. (b) Price (black line) versus discounted frequency (gray line),
with discount factor $\gamma=0.96$, for the same experiment as (a).}
\label{fig:adapt-price-glu-learn}
\end{figure}
In this section, we consider fractional Kelly agents who, as we saw in Section~\ref{sec:frac-kelly}, behave like full Kelly agents with belief $\lambda p + (1-\lambda)p_m$.
Figure~\ref{fig:adapt-price-glu-learn}.a graphs the dynamics of price in
an economy of 100 such agents, along with the observed frequency. Over
time, the price remains significantly more volatile than the frequency,
which converges toward $\pi=0.5$. Below, we characterize the transfer
of wealth that precipitates this added volatility; for now concentrate
on the price signal itself. Inspecting
Figure~\ref{fig:adapt-price-glu-learn}.a, price changes still exhibit a
marked dependence on event outcomes, though at any given period the
effect of recent history appears magnified, and the past discounted, as
compared with the observed frequency. Working from this intuition, we
attempt to fit the data to an appropriately modified measure of
frequency. Define the \emph{discounted frequency} at period $n$ as
\begin{equation}
d_n = \frac{\sum_{t=1}^n \gamma^{n-t} (1_{E(t)})}
           {\sum_{t=1}^n \gamma^{n-t} (1_{E(t)})  +
            \sum_{t=1}^n \gamma^{n-t} (1_{\overline{E(t)}}) },
\label{eq:discounted-frequency}
\end{equation}
where $1_{E(t)}$ is the indicator function for the event at period
$t$, and $\gamma$ is the \emph{discount factor}. Note that $\gamma=1$
recovers the standard observed frequency.

Figure~\ref{fig:adapt-price-glu-learn}.b illustrates a very close
correlation between discounted frequency, with $\gamma=0.96$ (hand
tuned), and the same price curve of
Figure~\ref{fig:adapt-price-glu-learn}.a. While standard frequency
provides a provably good model of price dynamics in an economy of
full-Kelly agents,
discounted frequency (\ref{eq:discounted-frequency}) appears a
better model for fractional Kelly agents.

%Prices are dependent on the apportionment of wealth among agents and
%also the agents' \emph{posterior} beliefs.

%Note that, in the current setting, the agents' beliefs themselves
%depend on price.
%
To explain the close fit to discounted frequency, one might expect that
wealth remains dispersed---as if the market's composite agent witnesses
fewer trials than actually occur.  That's true to an extent. Figure~\ref{fig:adapt-wealth-glu-learn-150} shows the distribution of wealth after 69 successes have occurred in 150 trials. Wealth is significantly more evenly distributed than a Beta distribution with parameters 69+1 and 81+1, also shown. However, the stretched distribution can't be modeled precisely as another, less-informed Beta distribution.
\begin{figure}
\includegraphics[scale=0.9]{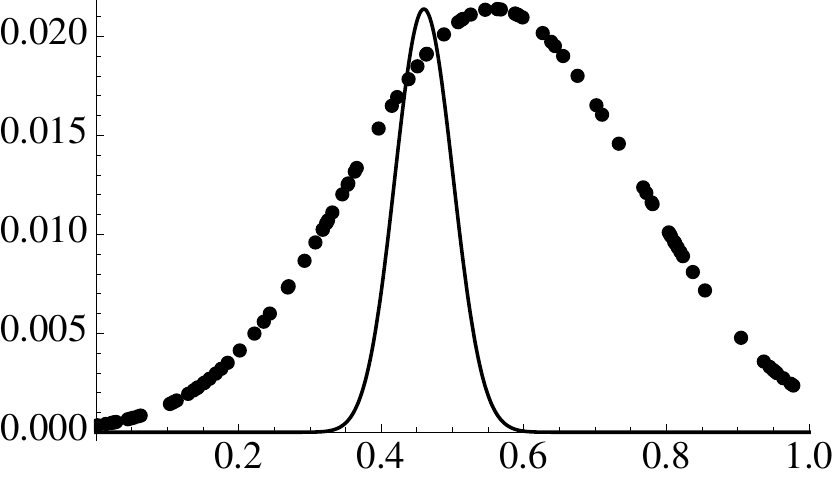}
\caption{(a) Wealth $w_i$ versus belief
$p_i$ at period 150 of the same experiment
as Figure~\ref{fig:adapt-price-glu-learn} with 100 agents with Kelly fraction $\lambda=0.2$. The observed frequency is
69/150
%(0.444)
and the solid line is $\mathtt{Beta}(69+1,81+1)$. The wealth distribution is significantly more evenly dispersed than the corresponding Beta distribution.}
\label{fig:adapt-wealth-glu-learn-150}
\end{figure}

\section{Learning the Kelly fraction}

In theory, a rational agent playing against rational opponents should set their Kelly fraction to $\lambda=0$, since, in a rational expectations equilibrium \cite{Grossman81}, the market price is by definition at least as informative as any agent's belief. This is the crux of the no-trade theorems \cite{Mil:82}. Despite the theory  \cite{Gea:82}, people do agree to disagree in practice and, simply put, trade happens. Still, placing substantial weight on the market price is often prudent. For example, in an online prediction contest called ProbabilitySports, 99.7\% of participants were outperformed by the unweighted average predictor, a typical result.\footnote{\tt\small http://www.overcomingbias.com/2007/02/how\_and\_when\_to.html}

In this light, fractional Kelly can be seen as an experts algorithm \cite{Cesa-Bianchi1997} with two experts: yourself and the market.
We propose dynamically updating $\lambda$ according to standard experts algorithm logic: When you're right, you increase $\lambda$ appropriately; when you're wrong, you decrease $\lambda$. This gives a long-term procedure for updating $\lambda$ that guarantees:
\begin{itemize}
\item You won't do too much worse than the market (which by definition earns 0)
\item You won't do too much worse than Kelly betting using your original prior $p$
\end{itemize}
For example, if you allocate an initial weight of $0.5$ to your
predictions and $0.5$ to the market's prediction, then the regret
guarantee of section~\ref{sec:sequence} implies that at most half of
all wealth is lost.

\section{Discussion}

We've shown something intuitively appealing here: self-interested
agents with log wealth utility create markets which learn to have
small regret according to log loss.  There are two distinct ``log''s
in this statement, and it's appealing to consider what happens when we
vary these.  When agents have some utility other than log wealth
utility, can we alter the structure of a market so that the market
dynamics make the market price have low log loss regret?  And
similarly if we care about some other loss---such as squared loss, 0/1
loss, or a quantile loss, can we craft a marketplace such that log
wealth utility agents achieve small regret with respect to these other
losses?

What happens in a market without Kelly bettors?  This can't be
described in general, although a couple special cases are relevant.
When all agents have constant absolute risk aversion, the market
computes a weighted geometric average of
beliefs~\cite{Pennock99-thesis,Pennock01-mfpo-tr,Rubinstein74}.  When
one of the bettors acts according to Kelly and the others in some more
irrational fashion.  In this case, the basic Kelly guarantee implies
that the Kelly bettor will come to dominate non-Kelly bettors with
equivalent or worse log loss.  If non-Kelly agents have a better log
loss, the behavior can vary, possibly imposing greater regret on the
marketplace if the Kelly bettor accrues the wealth despite a worse
prediction record.  For this reason, it may be desirable to make Kelly
betting an explicit option in prediction markets.

\bibliographystyle{abbrv}
\bibliography{kelly_betting} 

\begin{thebibliography}{10}

\bibitem{Breiman1961}
L.~Breiman.
\newblock Optimal gambling systems for favorable games.
\newblock In {\em Berkeley Symposium on Probability and Statistics, I}, pages
  65--78, 1961.

\bibitem{Cesa-Bianchi1997}
N.~Cesa-Bianchi, Y.~Freund, D.~Helbold, D.~Haussler, R.~Schapire, and
  M.~Warmuth.
\newblock How to use expert advice.
\newblock {\em Journal of the ACM}, 44(3):427--485, 1997.

\bibitem{Cover2006}
T.~M. Cover and J.~A. Thomas.
\newblock {\em Elements of Information Theory, Second Edition}.
\newblock Wiley-Interscience, New Jersey, 2006.

\bibitem{FSSW1997}
Y.~Freund, R.~Schapire, Y.~Singer, and M.~Warmuth.
\newblock Using and combining predictors that specialize.
\newblock In {\em Proceedings of the Twenty-Ninth Annual ACM Symposium on the
  Theory of Computing}, pages 334--343, 1997.

\bibitem{Gea:82}
J.~D. Geanakoplos and H.~M. Polemarchakis.
\newblock We can't disagree forever.
\newblock {\em Journal of Economic Theory}, 28(1):192--200, 1982.

\bibitem{Grossman81}
S.~J. Grossman.
\newblock An introduction to the theory of rational expectations under
  asymmetric information.
\newblock {\em Review of Economic Studies}, 48(4):541--559, 1981.

\bibitem{Kelly1956}
J.~Kelly.
\newblock A new interpretation of information rate.
\newblock {\em Bell System Technical Journal}, 35:917--926, 1956.

\bibitem{Mil:82}
P.~Milgrom and N.~L. Stokey.
\newblock Information, trade and common knowledge.
\newblock {\em Journal of Economic Theory}, 26(1):17--27, 1982.

\bibitem{Morris83}
P.~A. Morris.
\newblock An axiomatic approach to expert resolution.
\newblock {\em Management Science}, 29(1):24--32, 1983.

\bibitem{Pennock99-thesis}
D.~M. Pennock.
\newblock {\em Aggregating Probabilistic Beliefs: Market Mechanisms and
  Graphical Representations}.
\newblock PhD thesis, University of Michigan, 1999.

\bibitem{Pennock01-mfpo-tr}
D.~M. Pennock and M.~P. Wellman.
\newblock A market framework for pooling opinions.
\newblock Technical Report 2001-081, NEC Research Institute, 2001.

\bibitem{Rosenblueth92}
E.~Rosenblueth and M.~Ordaz.
\newblock Combination of expert opinions.
\newblock {\em Journal of Scientific and Industrial Research}, 51:572--580,
  1992.

\bibitem{Rubinstein74}
M.~Rubinstein.
\newblock An aggregation theorem for securities markets.
\newblock {\em Journal of Financial Economics}, 1(3):225--244, 1974.

\bibitem{Rubinstein75}
M.~Rubinstein.
\newblock Securities market efficiency in an {A}rrow-{D}ebreu economy.
\newblock {\em Americian Economic Review}, 65(5):812--824, 1975.

\bibitem{Rubinstein76}
M.~Rubinstein.
\newblock The strong case for the generalized logarithmic utility model as the
  premier model of financial markets.
\newblock {\em Journal of Finance}, 31(2):551--571, 1976.

\bibitem{Thorp1969}
E.~O. Thorp.
\newblock Optimal gambling systems for favorable games.
\newblock {\em Review of the International Statistical Institute}, 37:273--293,
  1969.

\bibitem{Thorp1997}
E.~O. Thorp.
\newblock The {K}elly criterion in blackjack, sports betting, and the stock
  market.
\newblock In {\em International Conference on Gambling and Risk Taking},
  Montreal, Canada, 1997.

\bibitem{Wolfers2006}
J.~Wolfers and E.~Zitzewitz.
\newblock Interpreting prediction market prices as probabilities.
\newblock Technical Report 12200, NBER, 2006.

\end{thebibliography}

\end{document}